\def\arXiv{}

\documentclass{article}
\usepackage{ijcai18}

\usepackage{times}
\usepackage{graphicx}
\usepackage{amsmath}
\usepackage{amssymb}
\usepackage{amsthm}
\usepackage{algorithm}
\usepackage{algpseudocode}
\usepackage{subcaption}
\usepackage[dvipsnames]{xcolor}
\usepackage{todonotes}
\usepackage{url}

\newcommand{\Oo}{\mathcal{O}}
\newcommand{\F}{\mathcal{F}}
\DeclareMathOperator{\E}{\mathbb{E}}
\newcommand{\expect}[2]{\E_{#1}\left[#2\right]}
\newcommand{\cf}{\text{cf}}
\newcommand{\R}{\mathbb{R}}
\newcommand{\lbin}{l_{0{\text -}1}}

\newcommand{\random}{$\mathtt{random}$}
\newcommand{\guided}{$\mathtt{guided}$}
\newcommand{\noncount}{$\mathtt{non-counterfactual}$}

\newcommand{\adult}{$\mathtt{adult}$}
\newcommand{\arrests}{$\mathtt{arrests}$}
\newcommand{\lendingclub}{$\mathtt{lending~club}$}

\newtheorem{defn}{Definition}
\newtheorem{thm}{Theorem}

\ifdefined\arXiv
\else
\fi

\ifdefined\draft
\newcommand{\gko}   [1]{{\color{Orange}{GK: #1}}}
\newcommand{\pxm}   [1]{{\color{Red}   {PM: #1}}}
\newcommand{\matt}  [1]{{\color{Green} {MF: #1}}}
\newcommand{\shayak}[1]{{\color{Blue}  {SS: #1}}}
\newcommand{\anupam}[1]{{\color{Maroon}{AD: #1}}}
\else
\newcommand{\gko}   [1]{}
\newcommand{\pxm}   [1]{}
\newcommand{\matt}  [1]{}
\newcommand{\shayak}[1]{}
\newcommand{\anupam}[1]{}
\fi

\title{Supervising Feature Influence}
\ifdefined\anon
\author{}
\else
\author{Shayak Sen, Piotr Mardziel, Anupam Datta, Matthew Fredrikson\\
Carnegie Mellon University
}
\fi

\begin{document}

\maketitle

\begin{abstract}
  Causal influence measures for machine learnt classifiers shed light
  on the reasons behind classification, and aid in identifying
  influential input features and revealing their biases.
  However, such analyses involve evaluating the classifier using
  datapoints that may be atypical of its training distribution.
  Standard methods for training classifiers that minimize empirical
  risk do not constrain the behavior of the classifier on such
  datapoints.
  As a result, training to minimize empirical risk does not
  distinguish among classifiers that agree on predictions in the
  training distribution but have wildly different causal influences.
  We term this problem \emph{covariate shift in causal testing} and
  formally characterize conditions under which it arises.
  As a solution to this problem, we propose a novel active learning
  algorithm that constrains the influence measures of the trained model.
  We prove that any two predictors whose errors are close on both the
  original training distribution and the distribution of atypical points
  are guaranteed to have causal influences that are also close.
  Further, we empirically demonstrate with synthetic labelers that our
  algorithm trains models that (i) have similar causal influences as
  the labeler's model, and (ii) generalize better to
  out-of-distribution points while (iii) retaining their accuracy on
  in-distribution points.
\end{abstract}

\section{Introduction}

Data processors employing machine learning algorithms are increasingly
being required to provide and account for reasons behind their
predictions due to regulations such as the EU GDPR~\cite{eu16gdpr}. This call
for reasoning tools has intensified with the increasing use of machine learning
systems in domains like criminal justice~\cite{angwin16propublica},
credit~\cite{Zest-finance}, and hiring~\cite{ideal}.  Understanding the reasons
behind prediction by measuring the importance or influence of attributes for
predictors has been an important area of study in machine learning.
Traditionally, influence measures were used to inform feature
selection~\cite{Breiman01}.  Recently, influence measures have received renewed
interest as part of a toolbox to explain operations and reveal biases of
inscrutable machine learning
systems~\cite{ribeiro2016lime,qii,adler2018auditing,kilbertus2017avoiding,datta2017proxy}.

Causal influence measures are a particularly important constituent of
this toolbox~\cite{qii,kilbertus2017avoiding}. By identifying attributes that
directly affect decisions, they provide insight about the operation of
complex machine learning systems. In particular, they enable identification of
principal reasons for decisions (e.g., credit denials) by evaluating
\emph{counterfactual queries} that ask whether changing input attributes would produce
a change in the decision. This determination is used to explain and
guard against unjust biases.
For example, the use of a correlate of age like income to
make credit decisions may be justified even if it causes applicants of
one age group to be approved at a higher rate than another whereas the direct
use of age or a correlate like zipcode may not be justified\footnote{This is an
example of a ``business necessity defense'' under US law on disparate
impact~\cite{burger71scotus}.}.

\begin{figure*}[t]
\begin{subfigure}{.5\textwidth}
\centering
\includegraphics[page=1,width=0.8\linewidth]{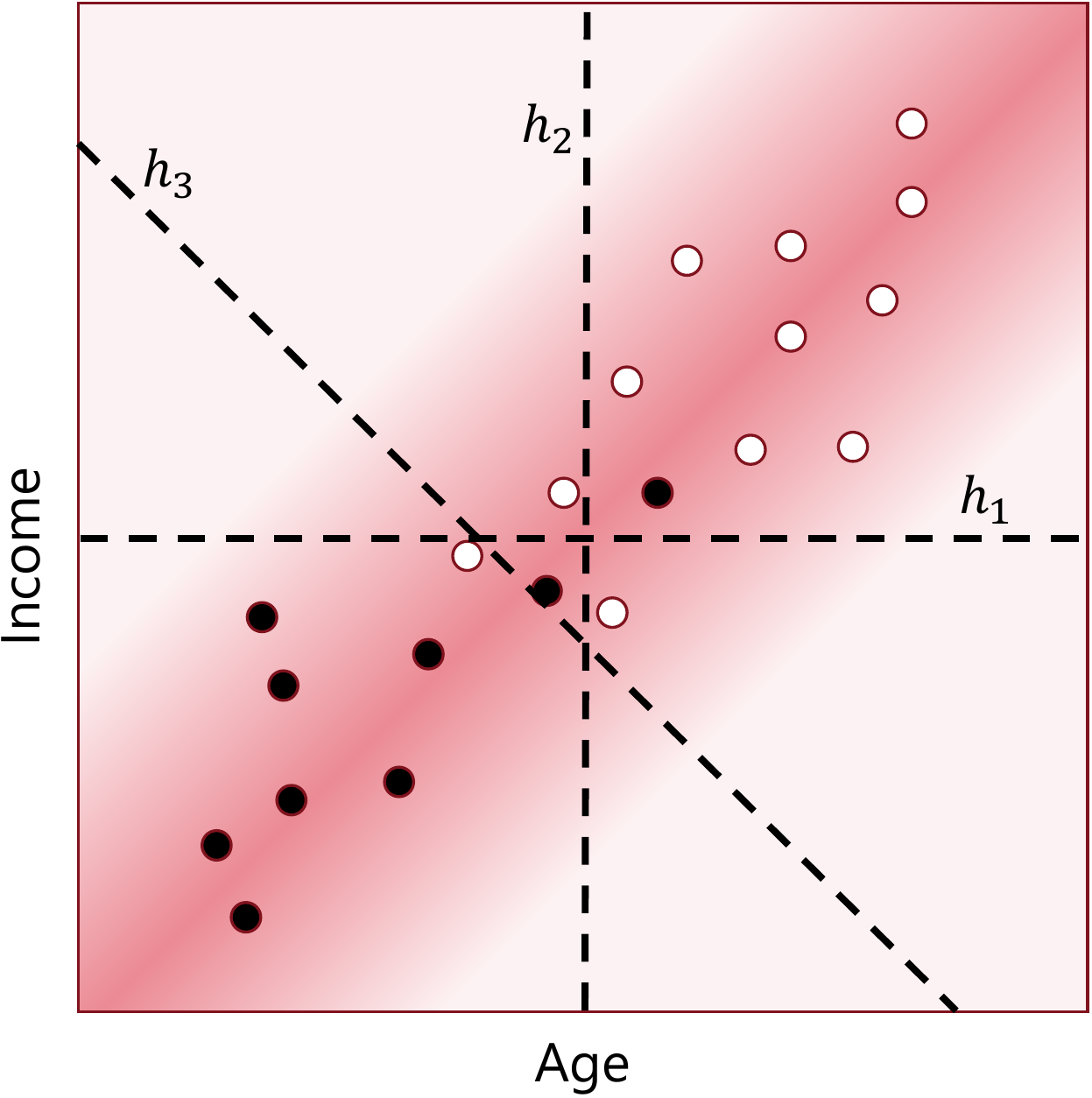}
\caption{Three predictors with similar in-distribution predictions}
\label{fig:dist-figures:1}
\end{subfigure}
\begin{subfigure}{.5\textwidth}
\centering
\includegraphics[page=2,width=0.8\linewidth]{figures/dist-figures-crop.pdf}
\caption{Causal testing for income}
\label{fig:dist-figures:2}
\end{subfigure}
\label{fig:dist-figures}
\caption{The three predictors trying to separate white points from black have similar predictions on the distribution, but very different causal influences. Predictor $h_1$ uses only income, $h_2$ uses only age, and $h_3$ uses a linear combination of the two. Figure (b) shows how counterfactual querying, by keeping age fixed and varying income allows the identification of causal influence and distinguishes between the three classifiers.}
\end{figure*}

Causal analyses of natural systems often involve observing outcomes of
specially created units, e.g.
mice with genes altered.
Such units may be atypical in natural populations.
However, while performing causal analysis over machine learnt systems,
a similar approach encounters an important challenge: machine learning
systems are not expected to be evaluated on atypical or
out-of-distribution units (datapoints), since they have not been
exposed to such units during training.
Standard methods for training classifiers that minimize empirical
risk do not constrain the behavior of the classifier on
such datapoints.
As a result, training to minimize empirical risk does not distinguish
among classifiers that agree on predictions in the training
distribution but have unintended causal influences.
We term this problem \emph{covariate shift in causal testing}.
In other words, typical machine learning algorithms are designed to make
the right predictions but not necessarily for justifiable reasons.

Returning to the example of credit decisions using age and income,
consider a situation where the two are strongly correlated: young
individuals have low income and older individuals have a higher income.
This situation is illustrated in Figure~\ref{fig:dist-figures:1} where
all three predictors $h_1, h_2, h_3$ have low predictive error, but
they make similar predictions for very different reasons.
Since the three predictors have nearly identical predictions on the
distribution, points from the distribution are not useful in
distinguishing the causal influence of the two features.
As a result, causal testing requires the creation of atypical units
that break the correlations between features.
For example, evaluating the predictor on the points on the red bar
(Figure~\ref{fig:dist-figures:2}) where age is fixed and income is
varied informs whether income is used by a given predictor or not.
However, since from an empirical risk minimization perspective the
atypical points are irrelevant, an algorithm optimizing just for
predictive accuracy is free to choose any of the three predictors.

We formally characterize conditions that give rise to covariate shift
in causal testing (Theorem~\ref{thm:constrain}).
Intuitively, this result states that if the units used for measuring
causal influence are sufficiently outside the data distribution,
constraining the behavior of a predictor on the data distribution does
not constrain the causal influences of the predictor.

In order to address this issue, we introduce an active learning algorithm in
Section~\ref{sec:cal}. This algorithm provides an accountability mechanism for data
processors to examine important features, and if their influences are suspicious,
to collect additional information that constrains the feature influences.
This additional information could steer the influences toward more
acceptable values (e.g., by reducing the influence of age in $h_2$ in Figure~\ref{fig:dist-figures:1}).
Alternatively, it could provide additional evidence that the influence values convey appropriate
 predictive power and the suspicions are unfounded (e.g., by preserving influences in
$h_1$ in Figure~\ref{fig:dist-figures:1}).

The active learning process is assisted by two oracles. The first is a feature
selection oracle that examines the causal influences of different features, and
chooses the feature for which counterfactuals queries should be answered.  We
envision this oracle to be an auditor who can identify problematic causal
influences based on background knowledge of causal factors or ethical norms
governing classification.  The second is similar to a standard active learning
oracle, and labels atypical points to answer counterfactual queries.  For
example, for predictor $h_2$ in our running example, the feature selection
oracle might notice that age has an unduly high influence, and can instruct the
algorithm to focus on instances that vary age while keeping income fixed. While
the direct use of age may be obviously problematic, in common applications the
system designer may not have apriori knowledge of which attribute uses are
problematic. The feature selection oracle may be able to spot suspiciously high
or low influences and guide the counterfactual queries that get sent to the
labeler to better inform the learning.

We evaluate the counterfactual active learning
algorithm for linear, decision tree, and random forest models on a number of
datasets, using a synthetic labeler. In particular, we demonstrate that after
counterfactual active learning, the trained classifier has similar causal
influence measures to the labeler. We also show that the classifier can
generalize better to out-of-distribution points. This is an important
consequence of having causal behavior similar to the labeler. Finally, we
demonstrate that the accuracy on the data distribution does not degrade as a
result of this additional training.

\paragraph*{Related Work.}
Prior work on causal learning learns the structure of causal
models~\cite{spirtes2000,hyttinen2013causal}, or given the structure of models,
the functional relationship between variables. In this context, active learning
has been used to aid both the discovery of causal
structures~\cite{tong2001active,he2008learning} and their functional
relationships~\cite{rubenstein2017active}. In this work we don't attempt to
learn true causal models. Instead, our work focuses on constraining the causal
behavior of learnt models. In doing so, we provide an accountability mechanism
for data processors to collect additional data that guides the causal
influences of their models to more acceptable values or justifies the causal
influences of the learnt model.

\paragraph*{Contributions.}
In summary, the contributions of this paper are as follows.
\begin{itemize}
  \item A formal articulation of the \emph{covariate shift in causal testing} problem.
  \item A novel active learning algorithm that addresses the problem.
  \item An empirical evaluation of the algorithm for standard machine learning predictors on a number of real-world datasets.
\end{itemize}

\section{Background}

\newcommand{\calX}[0]{{\cal X}}
\newcommand{\calY}[0]{{\cal Y}}

A predictor $h$ is a function $\calX \rightarrow \calY$ that operates
on an input space $\calX \subseteq \R^{n}$ to a space of predictions
$\calY$.
The input space $\calX$ has a probability distribution $P$ associated
with it, where $P(X = x)$ is the frequency of drawing a particular
instance $x$.

\subsection{Risk Minimization}

Given random variables $X \in \calX$, and $Y \in \calY$, and a loss
function $l$, the \emph{risk} associated with predictor $h$ is given
by
\[R(h) = \E\left[l(h(X), Y)\right].\] The goal of supervised learning
algorithms under a risk minimization paradigm is to minimize $R(h)$.
In general, the distributions over $X$ and $Y$ are unknown.
As a result, learning algorithms minimize \emph{empirical risk} over a
sample $\{(x_i, y_i)\}_{1..N}$
\[\hat{R}(h) = \frac{1}{N}\sum_{i=1}^N l(h(x_i), y_i).\]

Note that the risk minimization paradigm only constrains the behavior
of a predictor on points from the distribution and treats any two
predictors that have identical behavior on points from the
distribution interchangeably.

For ease of presentation, we focus on binary classification tasks where
$\calY$ is binary, and use the $0-1$ loss function
$\lbin(\hat{y}, y) = I(\hat{y}\not= y) = \left|\hat{y} - y\right|$.

\subsection{Counterfactual Influence}

The influence of a feature $f$ for a predictor $h$ is measured by
comparing the outcomes of $h$ on the data distribution to the outcomes
of a counterfactual distribution that changes the value of
$f$. We denote the data distribution over features as $X$ and the
counterfactual distribution with respect to feature $f$ as $X^f_\cf$.

A number of influence measures proposed in prior work can be viewed as
instances of this general idea.
For example, Permutation Importance~\cite{Breiman01}, measures the
difference in accuracy between $X$ and $X^f_\cf$, where $X^f_\cf$ is
chosen as $X$ randomly permuted.
In~\cite{adler2018auditing}, $X^f_\cf$ is chosen as the minimal
perturbation of $X$ such that feature $f$ cannot be predicted.
In this paper, we use Average Unary QII (auQII), an instance of
Quantitative Input Influence~\cite{qii}, as our causal influence
measure.
The counterfactual distribution $X^f_\cf$ for auQII is represented as
$X_{-f}U_{f}$, where the random variable $X_{-f}$ represents features except $f$ where  $U_f$ is sampled from the marginal
distribution of $f$ independently of the rest of the features $X_{-f}$.
\[P(X_{-f}U_f = x) = P(X_{-f} = x_{-f})P(U_f = x_f)\]

\begin{defn}
  Given a  model $h$, the Average Unary QII
  (auQII) of an input $f$, written $\iota_f(h)$, is defined as
  \[\iota_f(h) = \E_{X, U_f}\left[\lbin(h(X), h(X_{-f}U_f))\right]\]
\end{defn}

\section{Covariate shift in Causal Testing}
\label{sec:shift}

In this section, we discuss some of the theoretical implications of the
covariate shift in causal testing. First, we show in
Theorem~\ref{thm:constrain} that risk minimization does not constrain
influences when the data distribution diverges significantly from the
counterfactual distribution. In other words, predictors trained under an ERM
regime are free to choose influential factors.  Further, in
Theorem~\ref{thm:relating}, we demonstrate predictors that agree on
predictions on both the data distribution and the counterfactual distribution
have similar influences. This theorem forms the motivation for our
counterfactual active learning algorithm presented in Section~\ref{sec:cal}
that attempts to minimize errors on both the data and the counterfactual
distribution by adding points from the counterfactual distribution to the
training set.

\subsection{Counterfactual divergence}

We first define what it means for an influence measure to be unconstrained by its behavior
on the data distribution.
An influence measure $\iota_f$ is unconstrained for a predictor $h$, and data distribution
$X$, if it is possible to find a predictor $h'$ which has similar predictions
on the data distribution but very different influences. More specifically, if the influence is high,
then it can be reduced to a lower value, and vice versa.

\begin{defn}
  An influence measure $\iota_f$ is said to be
  $(\epsilon, \delta)$-unconstrained, for $0 \leq \delta \leq 1/2$,
  for a predictor $h$, if there exists predictors $h_1, h_2$ such
  that for $i \in \{1,2\}$, $P(h(X) \not= h_i(X)) \leq \epsilon$, and
  $\iota_f(h_1) \geq \delta$ and $\iota_f(h_2) \leq 1 - \delta$.
\end{defn}

The following theorem shows that if there exist regions $\varphi$ in
the input space with low probability weight in the data distribution
and high weight in the counterfactual distribution, i.e.
the data distribution and counterfactual distribution diverge
significantly, then any model will have unconstrained causal
influences.
As a result, predictors trained under an ERM regime are free to choose
influential causal factors.

\begin{thm}
  \label{thm:constrain}
  If there exists a predicate $\varphi$, such that $P(\varphi(X)) \leq \epsilon$ and $P(\varphi(X_{-f}U_f)\land \neg\varphi(X)) = \gamma$, then for any $h$,
  $\iota_f(h)$ is $(\epsilon, \gamma/2)$-unconstrained.
\end{thm}

\begin{proof}
  The proof proceeds via an averaging argument.
  Let $ \Pi $ be the set of all functions from $ \calX $ to
  $ \{0,1\} $.
  For $i \in \{1,2\}$ consider $h_i$ sampled uniformly from the set of
  deterministic functions that map values $ x $ satisfying $ \varphi $
  to $ h(x) $ and according to some $ \pi \in \Pi $ otherwise:
  $ \{ x \mapsto
  h(x) \text{ when } \neg \varphi(x) \text{ and } \pi(x) \text{ o.w.
  } \}_{\pi \in \Pi} $.
  Notice that $ P(h_i(X) | \neg \varphi(X)) $ is therefore uniform in
  $ \{0, 1\}$.

   As $h_i(x) = h(x)$ when $\neg\varphi(x)$, any such classifier
   satisfies $P(h(X) \not= h_i(X)) \leq P(\varphi(x)) \leq \epsilon$.
   Computing the expected influence over all $h_1$, we have
   \begin{align*}
     &  \expect{h_1}{\iota_f(h_1))}\\
   =&  \E_{h_1}\left[\E_{X, U_f}\left[I(h_1(X) \not= h_1(X_{-f}U_f))\right]\right] \\
    & \text{Let $\theta = \varphi(X_{-f}U_f)\land \neg\varphi(X)$. Then, $P(\theta) = \gamma$ }\\
   =&  \gamma \expect{X, U_f}{\expect{h_1}{I(h_1(X) \not= h_1(X_{-f}U_f))} ~|~ \theta}\\
   & +  (1-\gamma) \expect{X, U_f}{\expect{h_1}{I(h_1(X) \not= h_1(X_{-f}U_f))} ~|~ \neg\theta}\\
   &  \left(\text{if $\theta$, then  $\expect{h_1}{I(h_1(X) \not= h_1(X_{-f}U_f))} = \frac{1}{2}$}\right) \\
   \geq& \gamma\frac{1}{2} + (1-\gamma)0\\
   =&\gamma/2.
   \end{align*}
   By an averaging argument, there exists an $h_1^*$ such that  $\iota_f(h_1^*) \geq \gamma/2$.
   Similarly, computing the expected influence over all $h_2$, we have
   \begin{align*}
     &  \expect{h_2}{\iota_f(h_2))}\\
   =&  \E_{h_2}\left[\E_{X, U_f}\left[I(h_2(X) \not= h_2(X_{-f}U_f))\right]\right] \\
    & \text{Let $\theta = \varphi(X_{-f}U_f)\land \neg\varphi(X)$. Then, $P(\theta) = \gamma$ }\\
   =&  \gamma \expect{X, U_f}{\expect{h_2}{I(h_2(X) \not= h_2(X_{-f}U_f))} ~|~ \theta}\\
   & +  (1-\gamma) \expect{X, U_f}{\expect{h_2}{I(h_2(X) \not= h_2(X_{-f}U_f))} ~|~ \neg\theta}\\
   \leq& \gamma\frac{1}{2} + (1-\gamma)1\\
   =&1 - \gamma/2
   \end{align*}
   Again, by an averaging argument, there exists an $h_2^*$ such that  $\iota_f(h_2^*) \leq 1 - \gamma/2$
\end{proof}

\subsection{Relating counterfactual and true accuracies}
\label{sec:shift:relating}

We now show that if the two models agree on both the true and the counterfactual distributions, then they have similar influences.

\begin{defn}
  Given a loss function $l$ and predictors $h$ and $h'$, the expected
  loss of the $h$ with respect to $h'$, written $ err(h, h', X) $, is
\[err(h, h', X) = \E_{X}\left[\lbin(h(X), h'(X))\right].\]
\end{defn}

\begin{thm}
  \label{thm:relating}
  If $err(h, h', X) \leq \epsilon_1$, and $err(h, h', X_{-f}U_f) \leq \epsilon_2$, then $\left|\iota_f(h) - \iota_f(h')\right| \leq \epsilon_1 + \epsilon_2$
\end{thm}

\begin{proof}
\begin{align*}
  & \left|\iota(h, f) - \iota(h', f)\right| \\
  = & \left|\E_{X, U_f}\left[h(X) \not= h(X_{-f}U_f) \right]\right.\\
  &\hspace{5em}\left. -  \E_{X, U_f}\left[h'(X) \not= h'(X_{-f}U_f) \right]\right|  \\
  = & \left|\E_{X, U_f}\left[\left|h(X) - h(X_{-f}U_f)\right| \right]\right. \\
  &\hspace{5em}- \left.\E_{X, U_f}\left[\left|h'(X) - h'(X_{-f}U_f) \right|\right]\right|  \\
  & \text{by triangle inequality}\\
  \leq & \E_{X, U_f}\left[\left|h(X) - h(X_{-f}U_f)-  h'(X) + h'(X_{-f}U_f) \right|\right]  \\
  =    & \E_{X, U_f}\left[\left|h(X) -  h'(X) + h'(X_{-f}U_f) - h(X_{-f}U_f) \right|\right]  \\
  & \text{by triangle inequality} \\
  \leq    & \E_{X, U_f}\left[\left|h(X) -  h'(X)\right|\right]\\
  &\hspace{5em}+ E_{X, U_f}\left[\left|h'(X_{-f}U_f) - h(X_{-f}U_f) \right|\right]  \\
  = &  err(h, h', X) +  err(h, h', X_{-f}U_f) \\
  \leq & \epsilon_1 + \epsilon_2
\end{align*}
\end{proof}

\section{Counterfactual Active Learning}
\label{sec:cal}

\begin{figure*}[!htb]
\begin{subfigure}{0.01\textwidth}
  \rotatebox[origin=c]{90}{\adult}
\end{subfigure}
\begin{subfigure}{.33\textwidth}
\centering
\includegraphics[width=0.9\linewidth]{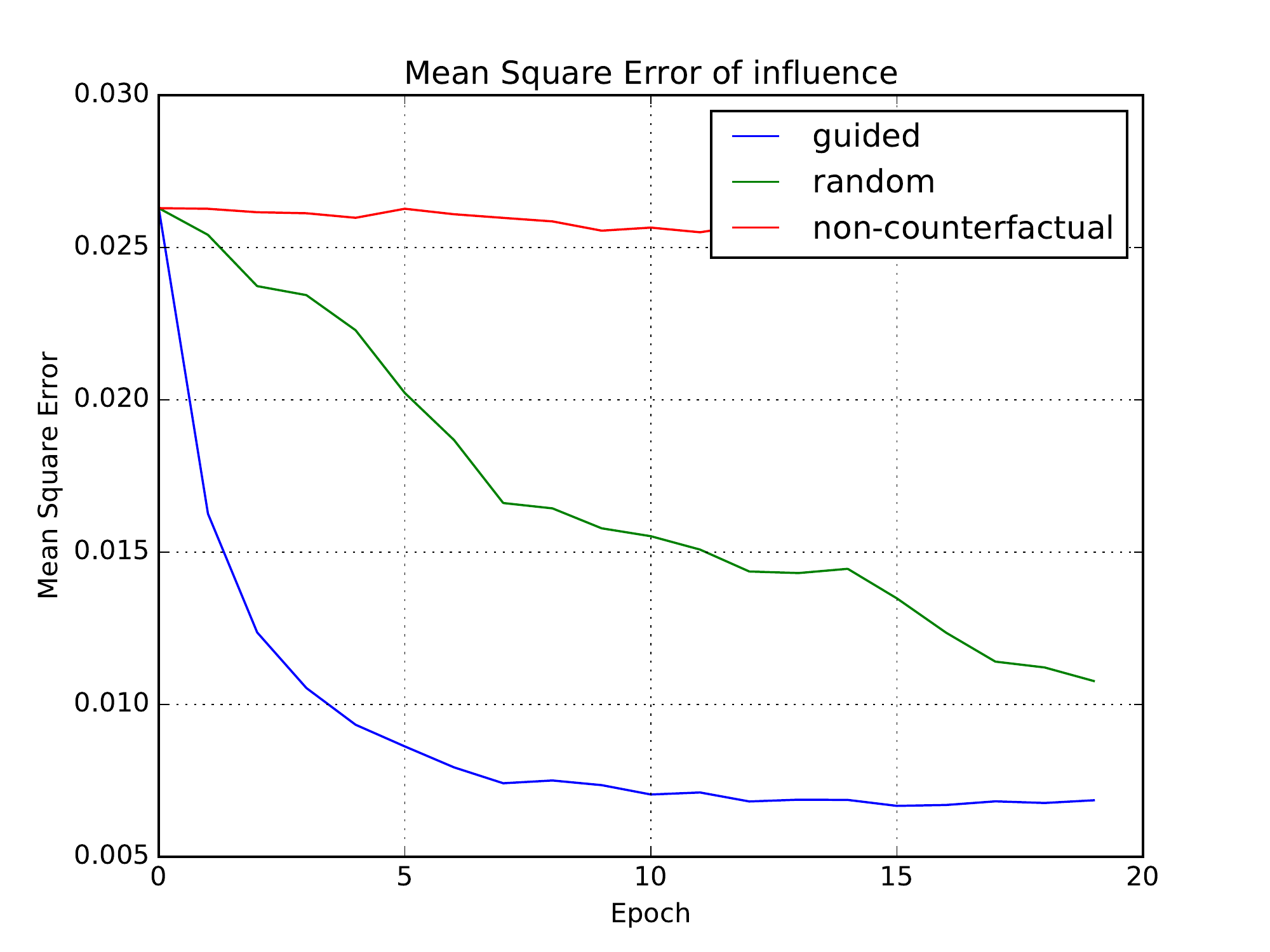}
\end{subfigure}
\begin{subfigure}{.33\textwidth}
\centering
\includegraphics[width=0.9\linewidth]{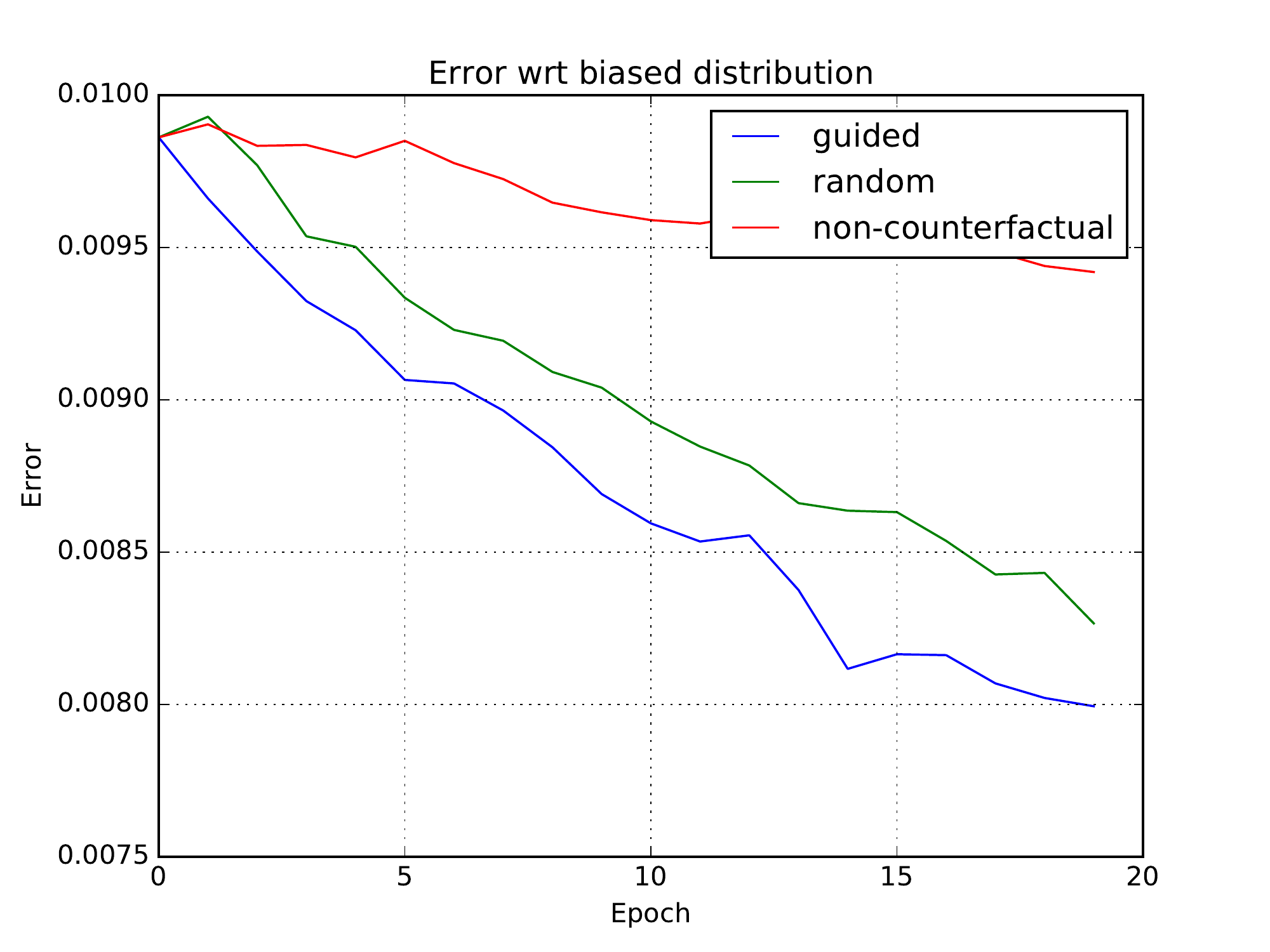}
\end{subfigure}
\begin{subfigure}{.33\textwidth}
\centering
\includegraphics[width=0.9\linewidth]{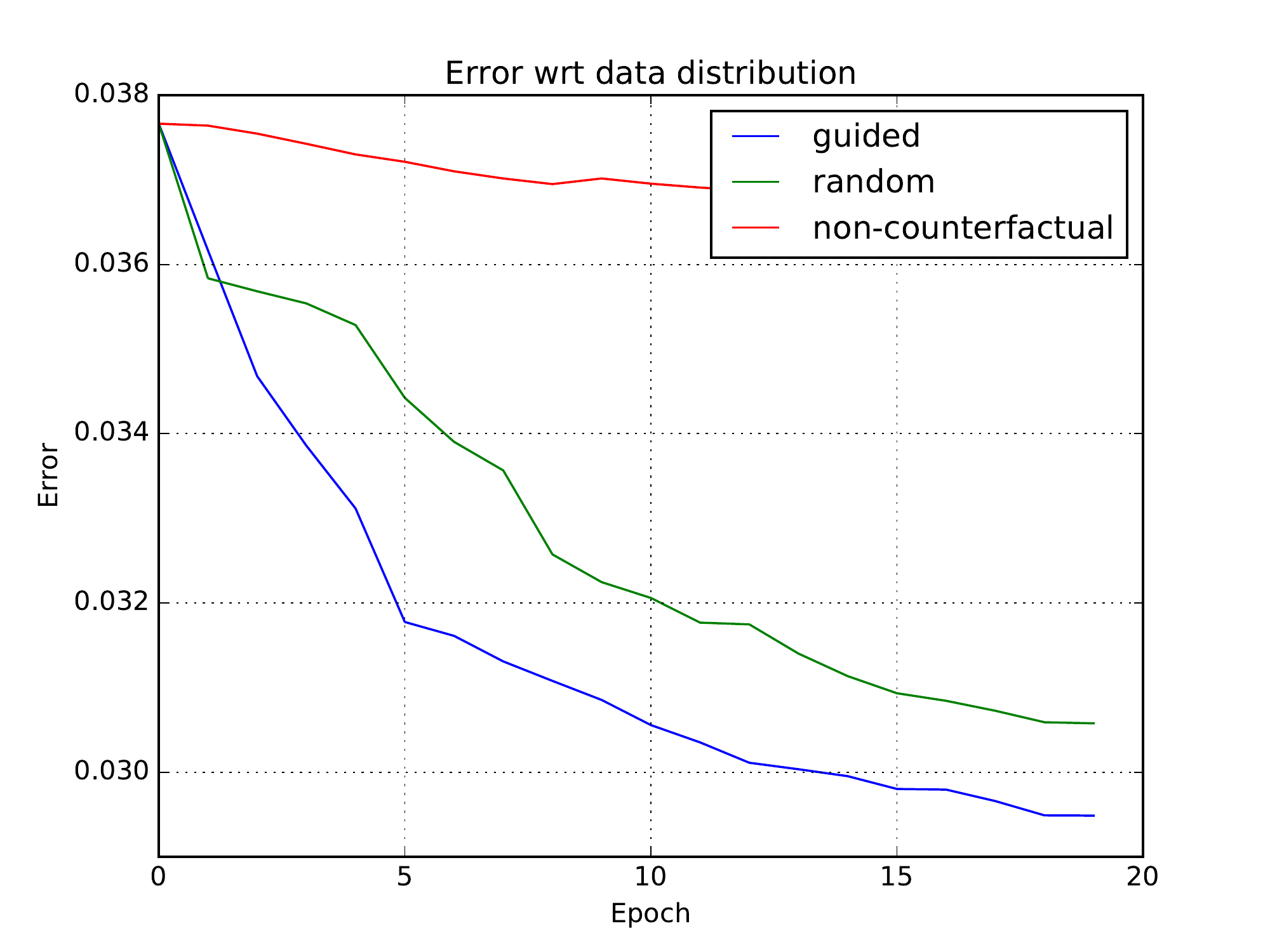}
\end{subfigure}

\begin{subfigure}{0.01\textwidth}
  \rotatebox[origin=c]{90}{\arrests}
\end{subfigure}
\begin{subfigure}{.33\textwidth}
\centering
\includegraphics[width=0.9\linewidth]{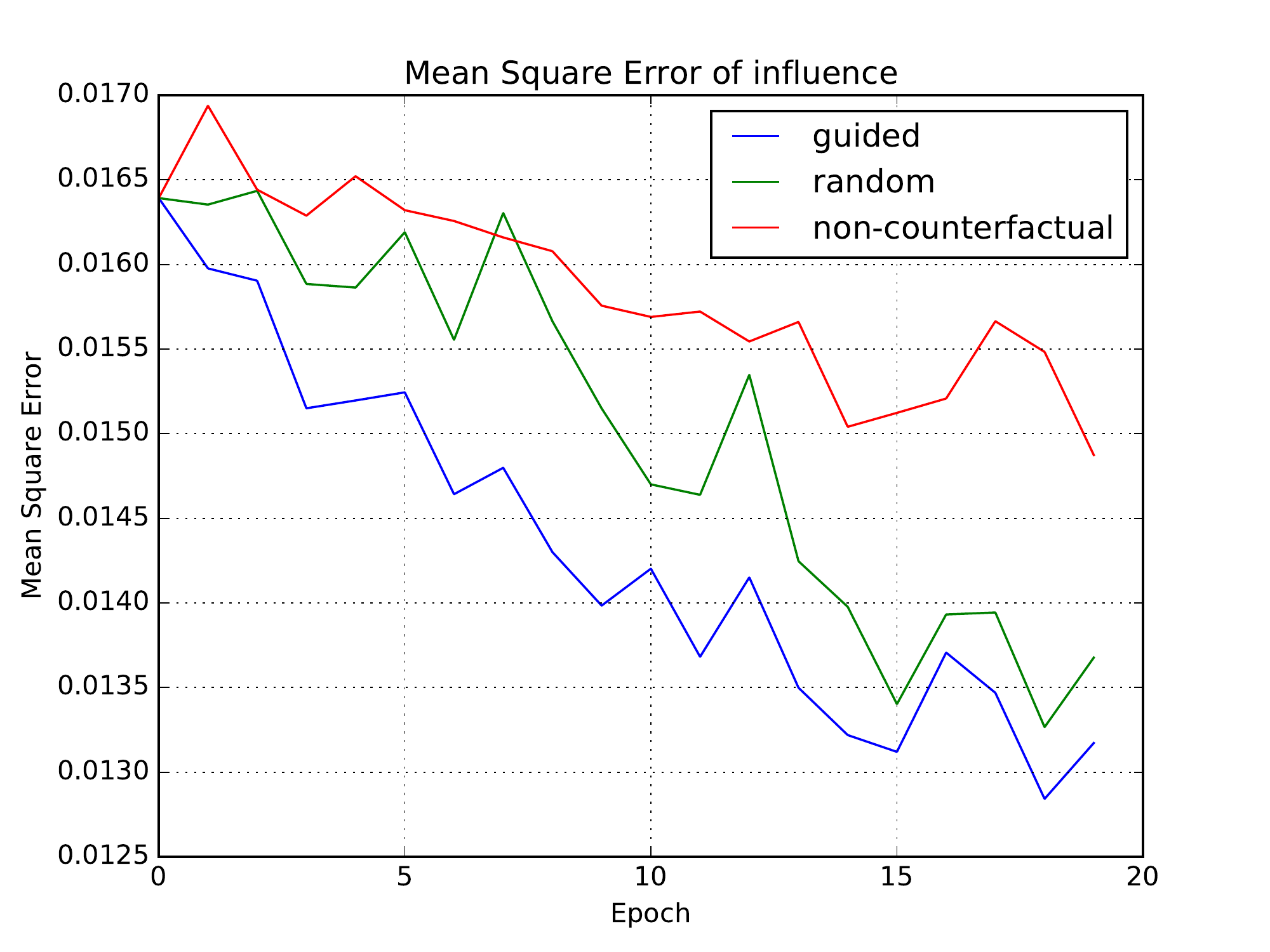}
\end{subfigure}
\begin{subfigure}{.33\textwidth}
\centering
\includegraphics[width=0.9\linewidth]{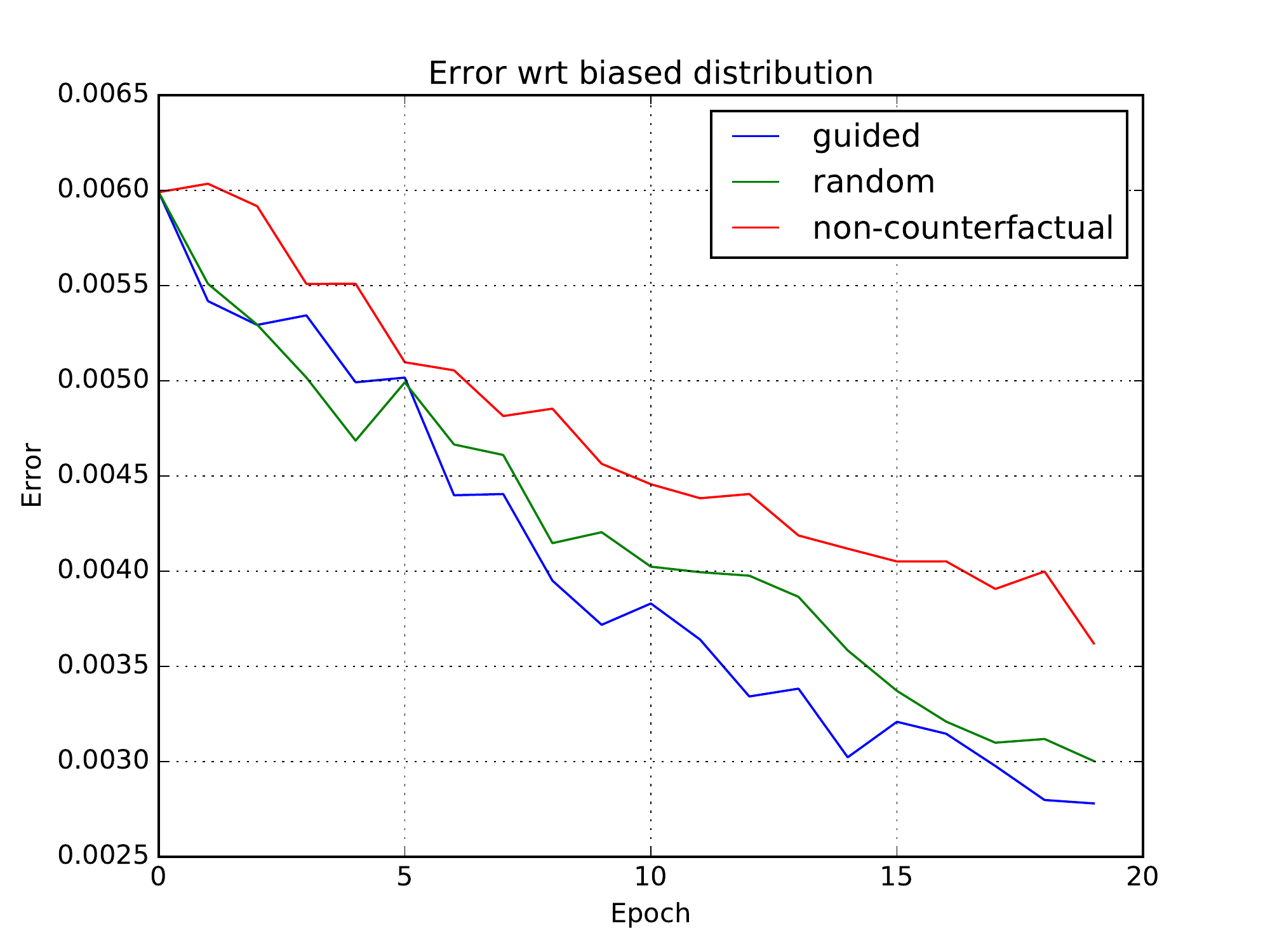}
\end{subfigure}
\begin{subfigure}{.33\textwidth}
\centering
\includegraphics[width=0.9\linewidth]{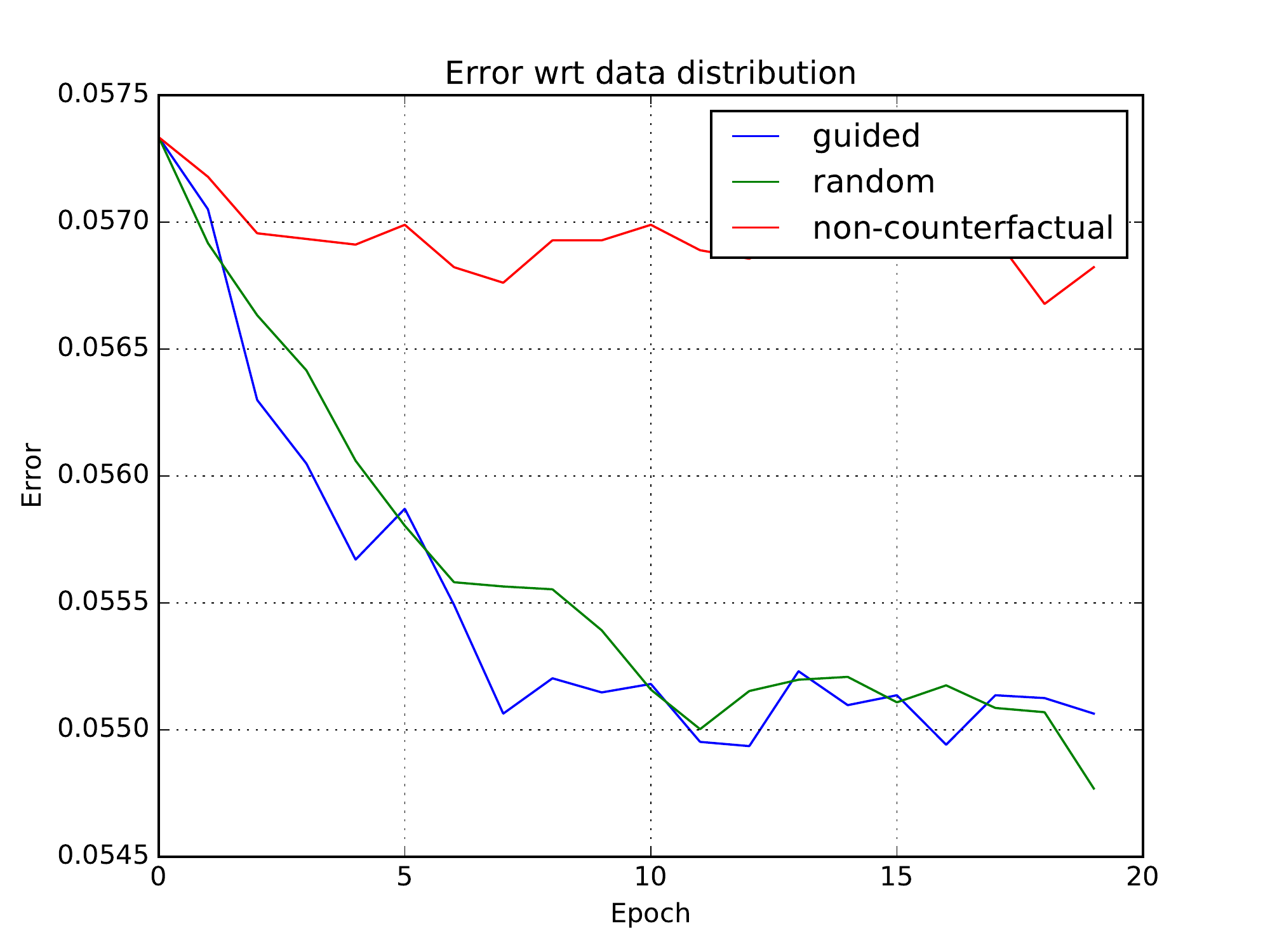}
\end{subfigure}

\begin{subfigure}{0.01\textwidth}
  \rotatebox[origin=c]{90}{\lendingclub}
\end{subfigure}
\begin{subfigure}{.33\textwidth}
\centering
\includegraphics[width=0.9\linewidth]{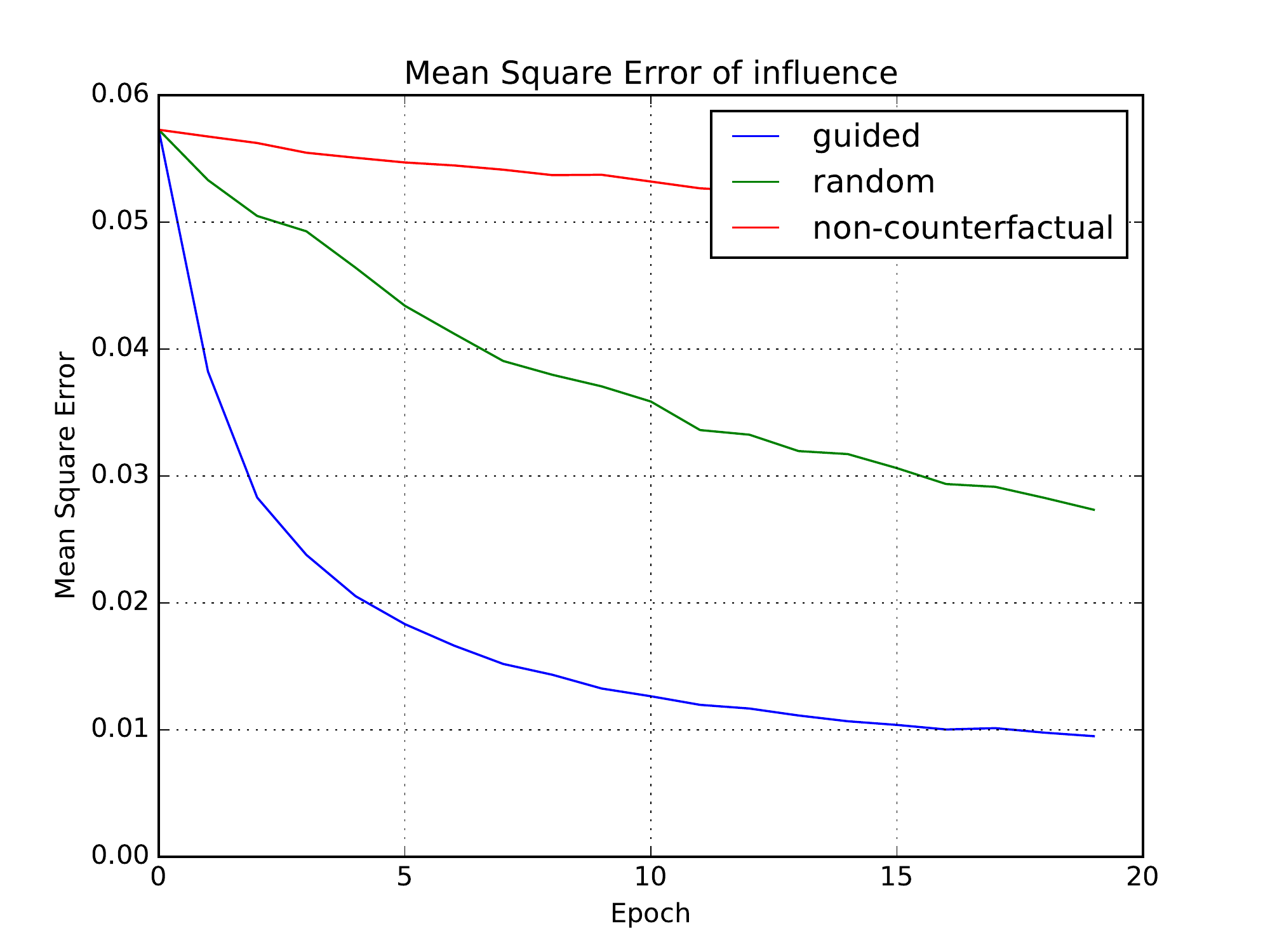}
\caption{Convergence of difference of influence}
\label{fig:linear-adult:inf}
\end{subfigure}
\begin{subfigure}{.33\textwidth}
\centering
\includegraphics[width=0.9\linewidth]{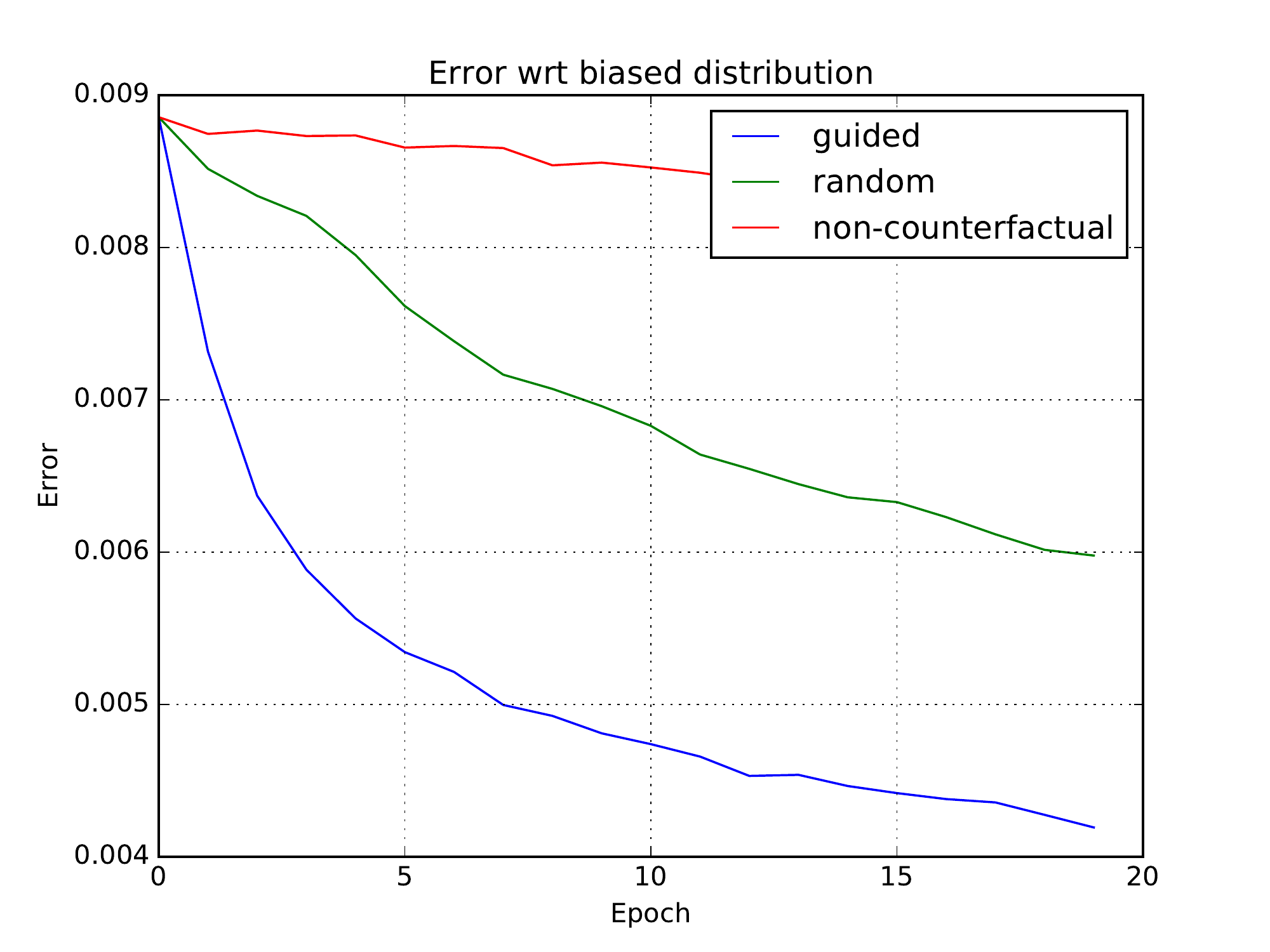}
\caption{Convergence of in-distribution error}
\label{fig:linear-adult:ide}
\end{subfigure}
\begin{subfigure}{.33\textwidth}
\centering
\includegraphics[width=0.9\linewidth]{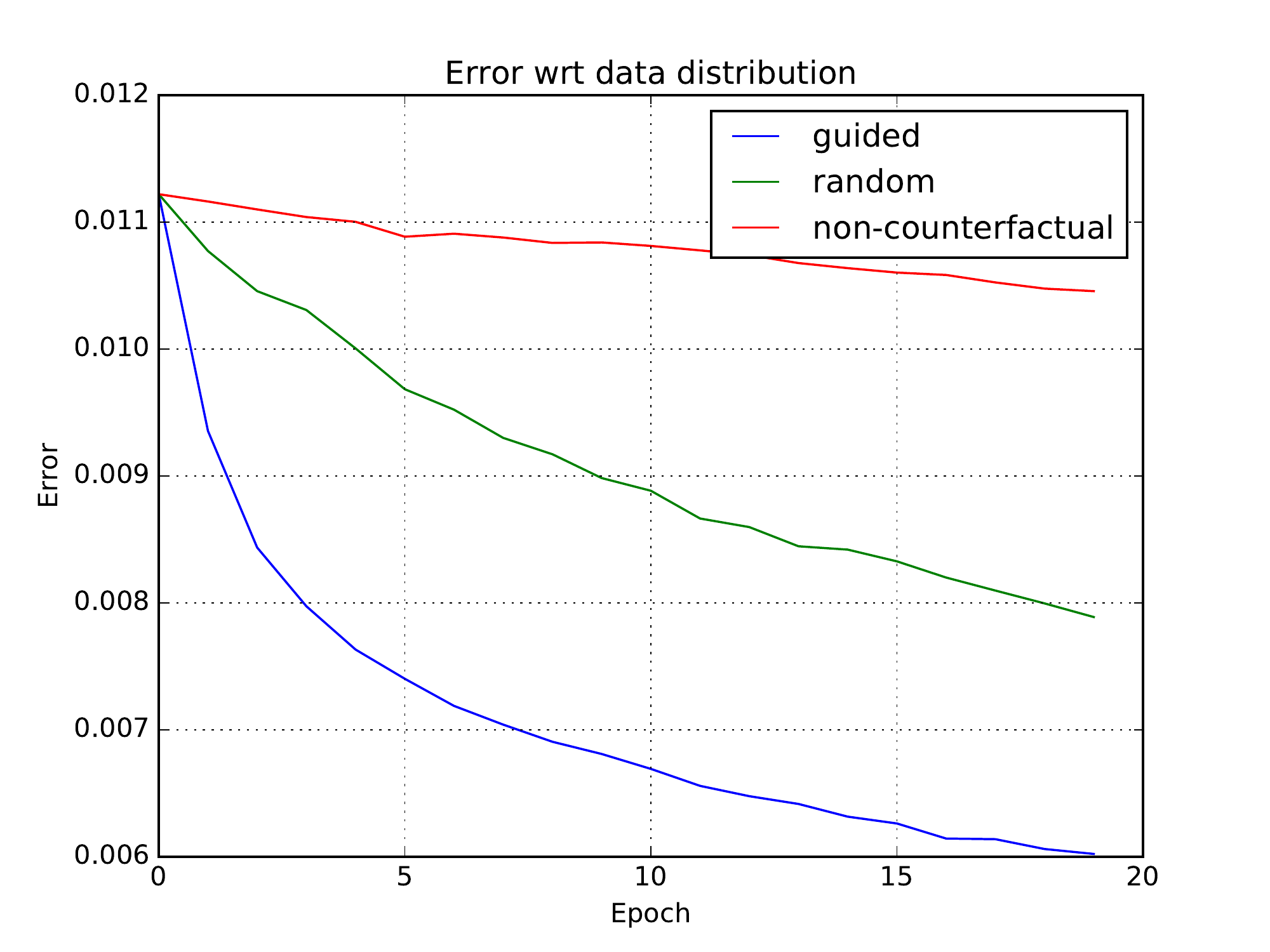}
\caption{Convergence of out of distribution error}
\label{fig:linear-adult:ode}
\end{subfigure}
\caption{Rates of convergence of counterfactual active learning with
  $\mathtt{guided}$, $\mathtt{random}$, and
  $\mathtt{non-counterfactual}$ settings.}
\label{fig:linear-adult}
\end{figure*}

\begin{figure}[!ht]
\begin{subfigure}{.5\textwidth}
\centering
\includegraphics[width=0.9\linewidth]{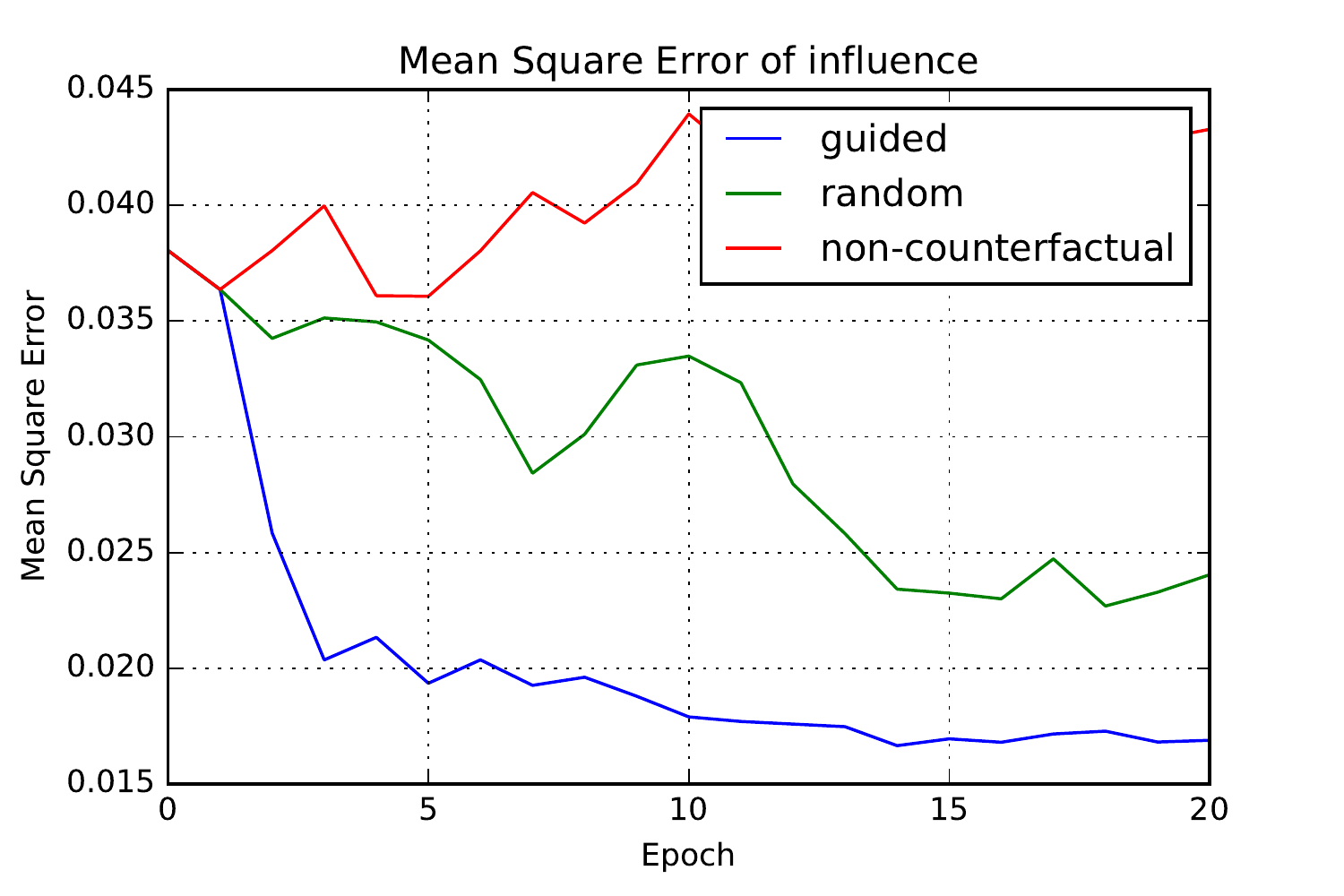}
\caption{Convergence of difference of influence}
\label{fig:tree-adult}
\end{subfigure}
\begin{subfigure}{.5\textwidth}
\centering
\includegraphics[width=0.9\linewidth]{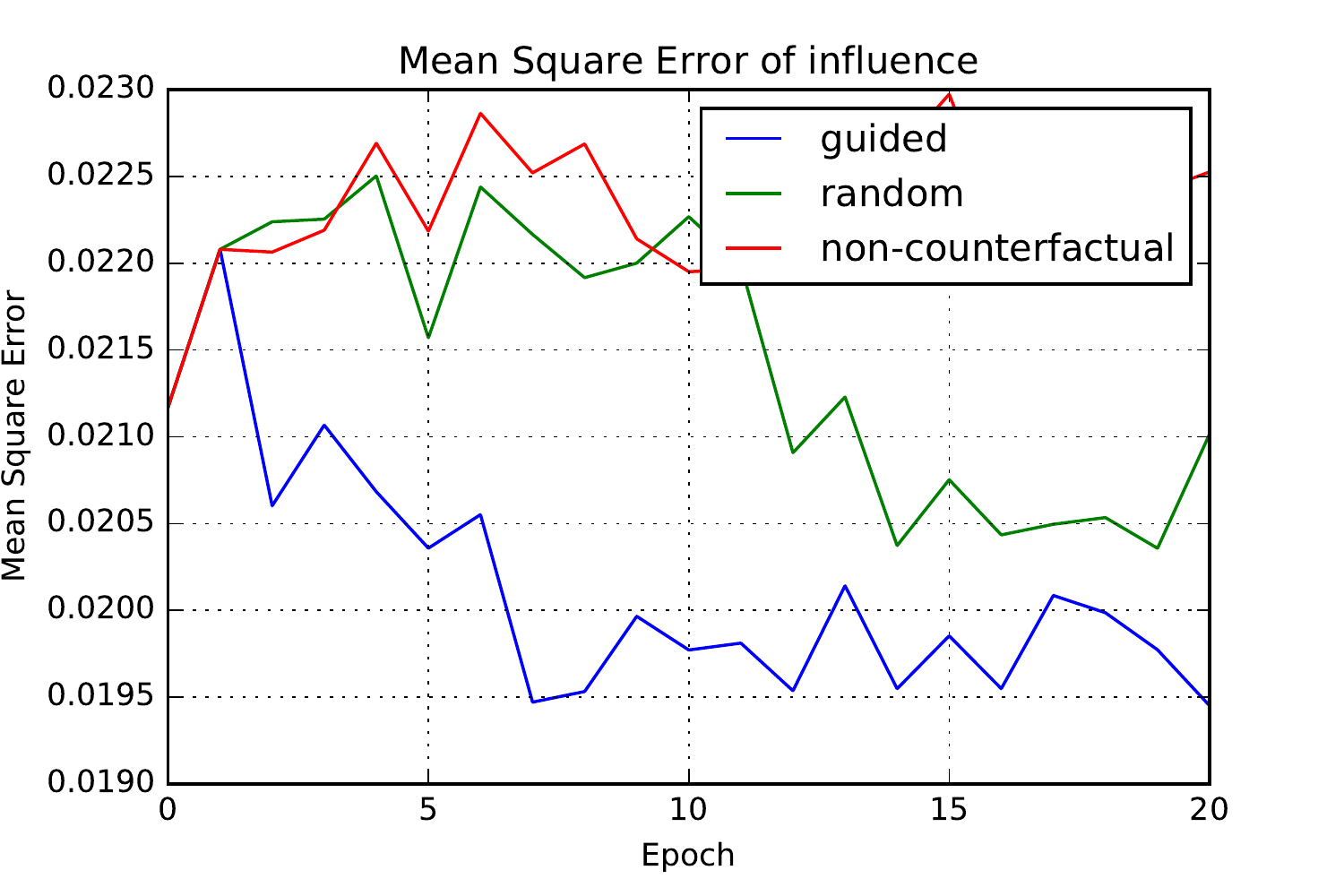}
\caption{Convergence of difference of influence}
\label{fig:forest-adult}
\end{subfigure}

\caption{Rates of convergence of counterfactual active learning with
  $\mathtt{guided}$, $\mathtt{random}$, and
  $\mathtt{non-counterfactual}$ settings.}
\end{figure}

In this section, we describe an active learning algorithm for training
a model that pushes the model towards the desired causal influences.
The learning is assisted by two oracles.
The first is a feature selection oracle that examines the causal
influences of input features, and chooses the feature for which
counterfactuals should be labeled.
We envision this oracle to be a domain expert that can identify
problematic causal influences based on background knowledge of causal
factors or ethical norms governing the classification task.
The second oracle, similar to a standard active learning oracle,
labels counterfactual points with their intended label.

\begin{algorithm}[ht]
  \caption{Counterfactual active learning.
    \label{alg:cal}}
  \begin{algorithmic}
    \Require Labeling oracle $\Oo$, Feature selection oracle $\F$
    \Procedure{CounterfactualActiveLearning}{$D, k$}
    \State $D$: training dataset
    \State $k$: labeling batch size
    \State $c \gets \text{train}(D)$
    \Repeat
    \State $\vec{\iota} \gets \text{feature influences}~QII(c, D)$
    \State $f \gets \text{feature}~\F(\vec{\iota})$
    \State $\hat{C} \xleftarrow[k]{\$} D_{-f}U_f$
    \State $\hat{y}\gets\text{oracle labels}~\Oo(\hat{C})$
    \State $D \gets D\cup\langle \hat{C}, \hat{y}\rangle$
    \State $c \gets \text{train}(D)$
    \Until stopping condition reached
    \State \Return $c$
    \EndProcedure
\end{algorithmic}
\end{algorithm}

The active learning process (Algorithm~\ref{alg:cal}), on every
iteration, computes the influences of features of a classifier trained
on the dataset.
The feature selection oracle $\F$ picks a feature.
Then $k$ points $\hat{C}$ are picked from the counterfactual
distribution, where $k$ is a pre-specified batch size parameter.
The parameter $k$ can also be thought of as a learning rate for the
algorithm.
The $k$ points in $\hat{C}$ are then labeled by the oracle $\Oo$ and
added to the training set.
A new classifier is trained on the augmented dataset and this process
is repeated until the stopping condition is reached.
The stopping condition can either be a pre-specified number of
iterations or a convergence condition when the classifier learnt does
not show a significant change in influences.

The choice of the feature selection oracle affect the speed of
convergence of the algorithm.
In our experiments, we consider two feature selection oracles (i) a
baseline \random~oracle, that picks features at random for generating
counterfactual queries, (ii) a \guided~oracle, that picks the feature
that has the highest difference in influence from the true influence.
In Section~\ref{sec:evaluation:results}, we demonstrate that an oracle
that deterministically picks the feature with the highest difference
in influence converges faster than an oracle that picks a feature at
random.

The rationale for training the classifier on points from the
counterfactual distribution is two-fold.
First, by adding points from the counterfactual distribution, the
algorithm reduces the divergence between the training distribution and
the counterfactual distribution, as a result, constraining the feature
influences of the learnt classifier according to
Theorem~\ref{thm:constrain}.
Additionally, by increasing accuracy of the classifier with respect to
the labeler on the counterfactual distribution, the influences of the
trained classifier are pushed closer to the influence of the
labeler (Theorem~\ref{thm:relating}).

\section{Evaluation}
\label{sec:evaluation}

In this section, we evaluate the counterfactual active learning
algorithm for linear, decision tree, and random forest models using a
synthetic labeler as ground truth.
In particular, we demonstrate that after counterfactual active
learning, the trained classifier has similar causal influence measures
to the labeler.
We also show that the classifier can generalize better to
out-of-distribution points.
This is an important consequence of having causal behavior similar to
the labeler.
And finally, we demonstrate that the accuracy on the data distribution
does not degrade as a result of this additional training.

\subsection{Methodology}

We evaluate our algorithm by training two predictors.
The first predictor provides the ground truth to be used by the
oracles.
The second predictor is trained on a biased version of the dataset
used to train the first predictor.
This approach induces a difference in influence in the two predictors.
In more detail, the following steps comprise our experimental
methodology.

\begin{itemize}
  \item Given: $D$ a dataset which is a sample from the original distribution,
  \item Train ground truth model $h_t$ on $D$. This model is used by the labeling oracle to respond to counterfactual queries.
  \item Select a random predicate $\theta$.
  \item Construct $D_b$ by excluding points from $D$ that satisfy $\theta$.
  \item Train predictor $h_b$ on a random subset of $D_b$, leaving the
    rest for testing and for use by the \noncount~baseline described
    below.
  \item Perform counterfactual active learning on predictor $h_b$.
\end{itemize}

The random predicate $\theta$ is chosen by training a short decision
tree on the dataset with random labels.
$D_b$ is intended to simulate a biased data collection mechanism in
order to induce artificial correlations in the dataset.
The induced artificial correlations create a gap between the
counterfactual distribution and the data distribution, thus making the
feature influences unconstrained.

We run the active learning algorithm under the following settings:
\begin{itemize}
\item \guided: At each iteration, the feature selection oracle selects
  the feature with the highest difference in auQII with respect to the
  base model.
\item \random: This is a baseline where the feature selection oracle
  selects a feature at random.
\item \noncount: This is another baseline where the labeling oracle
  labels fresh points from $D_b$ as opposed to from the counterfactual
  distribution.
\end{itemize}

The experiments presented here are run on the following datasets:
\begin{itemize}
  \item \adult: The benchmark adult dataset~\cite{uci-adult}, a subset
    of census data, is used to predict income from 13 demographic factors
    including age and marital status.
  \item \arrests: This data set is used to predict a history of arrests using 6 features
    extracted from the National Longitudinal Survey of Youth~\cite{nlsy97}
    such as drug and alcohol use.
  \item \lendingclub: A data set of loans originated by Lending Club~\cite{lendingclub} is
    used to predict charge-offs from 19 other financial variables about
    individuals.
\end{itemize}

These datasets represent prediction tasks that could be potentially used in settings such as predictive policing or credit, and where data processors are accountable for reasons behind prediction.

 All experiments presented here are run  with a batch size $t$ of $200$ for $20$
epochs, and averaged over $50$ runs of the algorithm.

\subsection{Results}
\label{sec:evaluation:results}

Figure~\ref{fig:linear-adult} shows the evolution of the active
learning algorithm for the three datasets dataset with a logistic regression
model.
In particular, Figure~\ref{fig:linear-adult:inf} shows the the change
of the mean square error of auQII between $h_b$ and $h_t$.
This figure shows that the feature influences converge to values close
to that of $h_t$ with the \guided~oracle.
The \random~oracle also converges but at a slower rate.
This result is useful as it indicates that the process does not
require the feature selection oracle to pick optimally.
The \noncount~algorithm does not affect the influence measures
significantly.
This is to be expected since it retrains using labeled points from
within the biased distribution.

In Figures~\ref{fig:linear-adult:ide} and ~\ref{fig:linear-adult:ode},
we show the accuracy of the classifier on holdout sets from $D_b$ and
$D$ respectively.
Figure~\ref{fig:linear-adult:ide} shows that the error on the data
distribution does not increase due to this additional training.
Further, Figure~\ref{fig:linear-adult:ode} shows that the error on the
unbiased dataset $D$ also decreases, even though parts of $D$ are not
in the training set.
This can be viewed as a  side-effect of the model becoming causally closer to the
ground truth model.
For all three datasets the \guided~oracle leads
to faster convergence across the three metrics.
The arrests dataset shows
similar behavior for the \random~and \guided~oracles which can be attributed to the dataset only containing a small number of features.

Figures~\ref{fig:tree-adult} and ~\ref{fig:forest-adult} show the effect of
counterfactual active learning on decision trees and random forests. They both
show a similar trend to linear models of the causal influences of $h_b$ converging
toward those of $h_t$ with counterfactual active learning.

\section{Conclusion and Future Work}

We articulate the problem of covariate shift in causal testing and
formally characterize conditions under which it arises. 
We present an algorithm for counterfactual active learning that addresses 
this problem. We empirically demonstrate with
synthetic labelers that our algorithm trains models that (i) have similar causal
influences as the labeler's model, and (ii) generalize better to out-of-distribution
points while (iii) retaining their accuracy on in-distribution points.

In this paper, we assume that the labeling oracle can label all points with equal
certainty and cost. However, for points further from the distribution, the
labeler might need to perform real experiments in order to label the points.
This suggests two interesting directions for future work. The first studies
mechanisms for answering counterfactual queries for points far away from the
distribution. The second involves designing an algorithm that takes into account 
the cost of a labeler in the learning process.

\ifdefined\anon
\else
\subsubsection*{Acknowledgments}
\begin{small}
  This work was developed with the support of NSF grants CNS-1704845
  as well as by the Air Force Research Laboratory under agreement
  number FA9550-17-1-0600.
  The U.S.
  Government is authorized to reproduce and distribute reprints for
  Governmental purposes not withstanding any copyright notation
  thereon.
  The views, opinions, and/or findings expressed are those of the
  author(s) and should not be interpreted as representing the official
  views or policies of the Air Force Research Laboratory, the National
  Science Foundation, or the U.S.
  Government.
\end{small}
\fi

\bibliographystyle{named}
\bibliography{bibs/blackbox.etc,bibs/causal-influence}

\end{document}